\documentclass[12pt,letterpaper]{article}

\usepackage{amsmath,amsthm,amssymb}
\newtheoremstyle{itheorem}    
  {}  
  {}  
  {\itshape}  
  {}  
  {\bfseries}   
  {.}     
  { }     
  {#1\if!#3!\else\ \fi\thmnote{#3}}
\newtheoremstyle{icorollary}  
  {}  
  {}  
  {}          
  {}  
  {\itshape}    
  {.---}  
  {0pt}   
  {#1}
\newtheoremstyle{numcorollary}
  {}  
  {}  
  {}          
  {}  
  {\itshape}    
  {.}     
  {}     
  {#1\if!#3!\else\ \fi\thmnote{#3}}
\newtheoremstyle{idefinition} 
  {}  
  {}  
  {}          
  {}  
  {\bfseries}   
  {.---}  
  {0pt}   
  {}
\newtheoremstyle{ilemma}      
  {}  
  {}  
  {\itshape}  
  {}  
  {\bfseries}   
  {.---}  
  {0pt}   
  {#1\if!#3!\else\ \fi\thmnote{#3}}

\theoremstyle{definition}
\newtheorem*{lemma*}{Lemma.}
\newtheorem{lemma}{Lemma.}[section]

\theoremstyle{definition}
\newtheorem*{theorem*}{Theorem.}
\newtheorem{theorem}{Theorem.}[section]

\theoremstyle{definition}
\newtheorem*{problem*}{Problem.}
\newtheorem{problem}{Problem.}[section]

\theoremstyle{icorollary}
\newtheorem*{corollary*}{Corollary.}

\theoremstyle{numcorollary}
\newtheorem*{ncorollary*}{Corollary.}

\theoremstyle{idefinition}
\newtheorem*{definition*}{Definition.}
\newtheorem{definition}{Definition.}

\newtheorem*{definitions*}{Definitions.}

\usepackage{hyperref}
\usepackage{cleveref}

\newcommand{\mIntegral}   [3]   {\ensuremath{\int_{#1}^{#2} \! {#3}\, \mathrm{d}}}
\newcommand{\nIntegral}   [3]   {\ensuremath{\int_{#1}^{#2} \! {#3}\, }}
\newcommand{\mNorm}       [1]   {\ensuremath{\left\lVert#1\right\rVert}}

\newcommand{\mSam}        [1]   {\ensuremath{\stackrel{#1}{\sim}}}
\newcommand{\asteq}{\mathrel{*}=}

\newcommand*{\xMin}{0}%
\newcommand*{\xMax}{9}%
\newcommand*{\yMin}{0}%
\newcommand*{\yMax}{9}%

\DeclareMathOperator{\expct}{\mathbf{E}}

\newcommand{\bbG}{\mathbb{G}}

\newcommand{\bbN}{\mathbb{N}}

\newcommand{\bbR}{\mathbb{R}}

\newcommand{\calG}{\mathcal{G}}

\newcommand{\calW}{\mathcal{W}}

\usepackage{rotating}
\usepackage{tikz}
\usetikzlibrary{
	arrows,
	backgrounds,
	cd,
	shadows,
	shapes,
	topaths,
	decorations.pathreplacing,
    decorations.text
}
\tikzset{
	position label/.style={
		below = 3pt,
		text height = 1.5ex,
		text depth = 1ex
	},
	brace/.style={
		decoration={brace, mirror},
		decorate
	}
}

\usepackage{adjustbox}

\tikzstyle{every picture}+=[remember picture]
\tikzstyle{na} = [shape=rectangle,inner sep=0pt]

\newcounter{scnt}
\newcounter{ecnt}
\newcommand{\myStrikeBegin}[1]{%
	\stepcounter{scnt}
	\tikz[baseline=(begin\the\value{scnt}.base)]\node[na](begin\the\value{scnt}){#1};
}
\newcommand{\myStrikeEnd}[1]{%
	\stepcounter{ecnt}
	\tikz[baseline=(end\the\value{ecnt}.base)]\node[na](end\the\value{ecnt}){#1};
	\begin{tikzpicture}[overlay]
		\draw (begin\the\value{ecnt}.west) -- (end\the\value{ecnt}.east);
	\end{tikzpicture}
}

\newcommand{\mydStrikeBegin}[1]{%
	\stepcounter{scnt}
	\tikz[baseline=(begin\the\value{scnt}.base)]\node[na](begin\the\value{scnt}){#1};
}
\newcommand{\mydStrikeEnd}[1]{%
	\stepcounter{ecnt}
	\tikz[baseline=(end\the\value{ecnt}.base)]\node[na](end\the\value{ecnt}){#1};
	\begin{tikzpicture}[overlay]
		\draw[->,red] (begin\the\value{ecnt}.west) -- (end\the\value{ecnt}.east);
	\end{tikzpicture}
}

\newcommand{\mArray}[3]{\ensuremath{[{#1}_{#2},\ldots,{#1}_{#3}]}}

\usepackage{xparse}

\ExplSyntaxOn
\NewDocumentCommand{\setlist}{mm}
{
	\clist_clear_new:c { l_jens_#1_array_clist }
	\clist_set:cn { l_jens_#1_array_clist } { #2 }
}
\NewExpandableDocumentCommand{\listitem}{mm}
{
	\clist_item:cn { l_jens_#1_array_clist } { #2 }
}
\ExplSyntaxOff

\usepackage{algorithm}
\usepackage{algpseudocode}

\colorlet{commentColor}{red}
\colorlet{kw}          {blue}         
\colorlet{type}        {pink}         

\algnewcommand{\Break}{\textcolor{kw}{break}}
\algnewcommand{\Continue}{\textcolor{kw}{continue}}
\algnewcommand{\algorithmicclass}{\textcolor{kw}{class}}
\algnewcommand{\algorithmicstruct}{\textcolor{kw}{struct}}
\algnewcommand{\String}{\textcolor{type}{string}}
\algnewcommand{\UniquePtr}[1]{\textcolor{type}{unique\_ptr<#1>}}
\algnewcommand{\Int}{\textcolor{type}{int}}

\algrenewcommand{\algorithmicfunction} {\textcolor{kw}{function}}
\algrenewcommand{\algorithmicwhile}    {\textcolor{kw}{while}}
\algrenewcommand{\algorithmicfor}      {\textcolor{kw}{for}}
\algrenewcommand{\algorithmicif}       {\textcolor{kw}{if}}
\algrenewcommand{\algorithmicthen}     {\textcolor{kw}{then}  }
\algrenewcommand{\algorithmicelse}     {\textcolor{kw}{else}}
\algrenewcommand{\algorithmicend}      {\textcolor{kw}{end}}

\algdef{SN} [STRUCT] {Struct} {EndStruct} [1]{\algorithmicstruct\ #1}
\algdef{SN} [CLASS]  {Class}  {EndClass}  [1]{\algorithmicclass\ #1}

\algdef{SE} [WHILE]  {While}  {EndWhile}  [1]{\algorithmicwhile\ #1}  {\algorithmicend~\algorithmicwhile}
\algdef{SE} [FOR]    {For}    {EndFor}    [1]{\algorithmicfor\ #1}    {\algorithmicend~\algorithmicfor}
\algdef{C}  [IF]     {IF}     {ElsIf}     [1]{\algorithmicelse\algorithmicif\ #1\ \algorithmicthen}

\algrenewcommand{\algorithmiccomment} [1] {\textcolor{commentColor}{\# #1}}

\makeatletter
\newcommand{\chapterauthor}[1]{%
	{\parindent0pt\vspace*{-25pt}%
		\linespread{1.1}\large\scshape#1%
		\par\nobreak\vspace*{35pt}}
	\@afterheading%
}
\makeatother

\usepackage{tabularx,booktabs}
\usepackage{url}
\usepackage{fullpage}

\begin{document}
	\thispagestyle{empty}
	
	\title{Efficient Evolutionary Models with Digraphons}
	
	\author
	{
		\centering
		Abhinav Tamaskar$^{1}$\footnote{Corresponding Author. Email: tamaskar@cims.nyu.edu} \and Bud Mishra$^{1}$ \\
		$^{1}$Courant Institute of Mathematical Sciences, NYU
	}
	
	\date{}
	
	\maketitle
	
	\begin{abstract}
	We present two main contributions which help us in leveraging the theory of graphons for modeling evolutionary processes. We show a generative model for digraphons using a finite basis of subgraphs, which is representative of biological networks with evolution by duplication\cite{ohno2013evolution}. We show a simple MAP estimate on the Bayesian non parametric model using the Dirichlet Chinese restaurant process representation, with the help of a Gibbs sampling algorithm to infer the prior. Next we show an efficient implementation to do simulations on finite basis segmentations of digraphons. This implementation is used for developing fast evolutionary simulations with the help of an efficient 2-D representation of the digraphon using dynamic segment-trees with the square-root decomposition representation. We further show how this representation is flexible enough to handle changing graph nodes and can be used to also model dynamic digraphons with the help of an amortized update representation to achieve an efficient time complexity of the update at $O(\sqrt{n}\log{n})$, where $n$ is the number of nodes in the digraph.
	\end{abstract}
	
	\section{Introduction}
Graphical models are one of the most important tools used in machine learning\cite{koller2009probabilistic} and arise in most applications which involve pairwise interactions, such as mutations in cancer evolution\cite{zhao2019cancer}, protein networks\cite{grzegorczyk2007extracting, werhli2006comparative}, hierarchical network models\cite{dempsey2021hierarchical}, influence in social networks\cite{mason2007situating,robins2001random,goyal2010learning}, population dynamics\cite{boer2012modeling} and many more. In machine learning, there are various techniques which forgo the use of these graphs and instead employ more algebraic representations to take advantage of the underlying theories, such as latent variable models\cite{loehlin1987latent}, network or dynamic models\cite{levin1976population}, deep neural networks (DNN)\cite{uppu2017tuning,miikkulainen2019evolving}, clustering models\cite{brannstrom2005role,chen2007bayesian}. The key advantage of the later techniques is the reliance of the abundance of techniques developed in linear algebra and optimized algorithms for doing fast implementations to get efficient real world analysis.

One of the main motivations of this study is the case of evolutionary populations, where the evolution is modeled as interactions between individuals of the populations, such as mutations, genotypic variations and phenotypic selection. In such cases, evolution of the gene regulatory network (GRN) or the protein-protein interaction (PPI) network happens by specific events, such as insertion, deletion, duplication, point mutations, translocation and inversion\cite{juan2008co,erwin2009evolution}. Of these the insertion, deletion and duplication events have the most noticeable effect on the networks and have an easily observable effect on the phenotype. If we think of genes as nodes in a graph and gene interactions as edges, these events can be thought of as an edge or node insertion, deletion or duplication.

When our network starts evolving and growing in size, we naturally think about what the outcome of such a process would be. As our network keeps on increasing in size, we need to extend our definition of a graph and naturally arrive at the intuition of a \emph{limit graph}, i.e. a graph on an infinite number of nodes, which we analyze not by looking at properties on non-empty subsets of the graph. These dynamics can be represented in terms of limit networks with the help of \emph{digraphons}. A \emph{digraphon} is measurable function $G:[0,1]^2\to[0,1]$. Given a digraphon $G$, there is a corresponding countably infinite exchangeable graph $\calG(\bbN,G)$, with the adjacency matrix $(\calG_{ij})_{i,j\in \bbN}$ defined by the generative model
\begin{align*}
    U_i &\mSam{\text{iid}} \text{Uniform}[0,1]~~ \forall ~ i\in \bbN \\
    \bbG_{ij}|U_i,U_j & \mSam{\text{ind}} \text{Bernoulli}(G(U_i, U_j))
\end{align*}

Thus the digraphons are an ideal object to use for a generative model for evolutionary networks. Any useful tool which is to be used for real world analysis must also support \emph{hypothesis refutability}, which necessitates the notion of similarity between digraphons. The space of digraphons have many norms defined on it, such as the $l_p$ norm
\[ \mNorm{G}_p = \left(\nIntegral{(i,j)\in[0,1]^2}{}{(G_{ij})^p}\right)^{\frac{1}{p}}\]
or the more interesting, \emph{cut norm}, which can be thought of as the maximum dissimilarity in a bounded region,
\[ \mNorm{G}_\square = \sup_{S,T\subseteq[0,1]}\nIntegral{S\times T}{}{G}\]

The cut metric is quite complex to calculate and can be approximated in finite real world scenarios with the \emph{agony} metric which is exponentially faster to compute\cite{tatti2014faster}.

These techniques for analyzing digraphons have recently been developed and have yet to see a wider use in conjunction with the standard Bayesian statistical tools prevalent in machine learning. The notions of limit graphs and asymptotic behaviours of evolutionary models are very important in using generative models from Bayesian statistics as the above model for digraphons seems to suggest an intuitive method for reasoning an approximation of the parameters.

The current black box learning from DNNs falls short of generating an explainable hypothesis which is needed for refutability. Indeed such scenarios have previously been observed such as the universal adversarial perturbations, which have then been leveraged to design adversarial networks. But these still fall short of an ideal answer.

The importance of these tools has been seen in many places, such as those used for signaling games\cite{lacroix2020evolutionary}, population dynamics and biomolecular networks, mesh network topologies\cite{de2007coverage}, 3-D neural imaging reconstruction\cite{hu2019topology}, etc. These techniques work directly in complement to the notions from deep neural networks by providing an explainable AI which can then be used as a hyperparameter in designing the DNNs by affecting their layer hierarchies, activation functions, dropout scenarios and memory length estimates.
	\section{Background - Graphons and Digraphons}

Let $\calW$ denote the space of all symmetric, bounded, measurable functions $\calW:[0,1]^2\to \bbR$. This is defined as the set of all \emph{kernels} , reminiscent of the kernels used in Support Vector Machines. If we restrict our attention to the set of functions $W\in \calW_0$ such that $0\le W\le 1$, we arrive at the space of \emph{graphons}. If we drop the condition of symmetry, we arrive at the space of digraphons.

For our purposes, we do not distinguish between the functions which are equal almost everywhere, as with most analytical scenarios, it is not possible to distinguish between such functions. We will soon see that this is actually not the only equivalence we want to put on the functions, as we will need to also equate a larger class of functions to each other for the sake of \textbf{exchangeability}, which is important for statistical modeling. 

The notion of graphons is an important one where we want to look at limits of graph sequences. Such graph sequences arise in a variety of different natural scenarios, such as social networks, recommendation systems for advertisements, shopping, etc., and also in biological scenarios, such as genetic mutations, evolutionary models, population dynamics. As a general rule of thumb, any scenario where we have a dynamic population with potential for growth in event space is a candidate for getting sequence of graphs.

Natural questions arise in how to analyze such a sequence. Does the growth follow a pattern? Are there noticeable features of this graph that are preserved across its growth? Does the graph sequence converge to any discernible end object? The theory of graphons (and digraphons) tries to answer many such questions in a rigorous form.

For us to have a notion of convergence and similarity, we need to start with a notion of a distance between digraphs. There are many distances defined on the space of digraphs, the distances introduced by the $L_p$ norms, nuclear norms, etc.; we will restrict our attention to the more interesting \emph{cut distance}.

\begin{definition}[Cut Distance]
 For two directed graphs $G, G'$, on the same set of vertices $V$, the cut distance is defined as 
 \[ d_\square(G, G') = \max_{S, T\subseteq V}\frac{e_G(S, T) - e_{G'}(S, T)}{|V|^2}\]
\end{definition} 

where $e_G(S, T)$ denotes the number of edges between $S$ and $T$ in the graph $G$.

If we let $d_1(G, G')$ be the $L_1$ distance on the adjacency matrices of $G$ and $G'$ we get the inequality $ d_\square(G, G') \le d_1(G, G') $. Hence we see that the two distances give differing information. As an example, for two Erdos-Renyi graphs with $p=1/2$, we get that $E[d_1(G ,G')] = 1/2$ while $E[d_\square(G, G')] = \theta(1/\sqrt{n})$. 

For unlabeled graphs on the same set of nodes, the intuitive extension to the cut distance is to define it via equivalence on node relabelings, which turns out to be the correct one. Let $G, G'$ be two graphs on same number of nodes, then the cut distance is overloaded as

\[ d_\square(G, G') = \min_{\phi\in\text{Hom}(G, G')}d_\square(\phi(G), G').\] 

Here we minimize over all homomorphims of $G$ into $G'$. For generalized measurable digraphons, we first define the \emph{cut norm}

\[ \mNorm{W}_\square = \sup_{S, T\subseteq[0,1]}\nIntegral{S\times T}{}{W}\]

Which can then be extended to the cut distance as $d_\square{W, W'} = \mNorm{W-W'}_\square$. Similar to the finite case, we achieve the inequalities between norms
\[ \mNorm{W}_\square \le \mNorm{W}_1 \le \mNorm{W}_p \le \mNorm{W}_{\infty} \le 1\]

Again, similar to the finite case of the cut distance, where we have ``relabelings'' via measure preserving homomorphisms of $\phi:[0, 1]\to[0 ,1]$.

 \[ d_\square(W, W') = \inf_{\phi\in\text{Hom}([0,1])} d_\square(\phi(W), W') \]
 
 It is important to note  that finite graphs can be represented as a specific case of a graphon by using step functions. For example, a digraph on $[n]$ with the adjacency matrix $A_{i,j}$ can be canonically viewed as a digraphon $W$, with $W(\frac{i}{n}, \frac{j}{n}) = A[i, j]$. This allows us to treat even finite graphs as a graphons and simplify our analysis, where we no longer have to distinguish between sequences of graphs vs sequence of graphons.
 
We see that any measure preserving transformation of a digraphon has zero cut-distance to the original.  An important theorem states that the only digraphons which have cut distance zero are the ones under measure preserving transformations of the original or of one which is equal almost everywhere, termed as weakly isomorphic pairs.

\begin{theorem}[Weak isomorphism theorem\cite{lovasz2012nondeterministic}]
    Let $W, W'$ be two digraphons, then $d_\square(W, W') = 0$ if and only if there exists a digraphon $Z$, such that $W = Z$ almost everywhere and $W'$ is a measure preserving homomorphism of $Z$\footnote{In particular this also covers the case where $W=Z$ everywhere.}.
 \end{theorem}
  
This result is one of the most important ones for the analysis of digraphons, as this gives confidence to our sampling algorithms. Indeed, the proof of this theorem itself relies on the canonical digraphon and sampling state introduced. Then we generalize the distance for two digraphons $W, W'$ by round robin chasing across the commutative diagram, 
\[d_\square(W, \text{Sample}(W)) \leftrightarrow d_\square(W, W') \leftrightarrow d_\square(W', \text{Sample}(W'))\]

\begin{lemma}[Convergence in norm]
	Let $W_n, n= 1,2,\ldots,$ be a sequence of digraphons such that $\mNorm{W_n}_\square\to 0$. Then for all dikernels $Z$, we have that $\mNorm{W_nZ}_\square\to 0$.
\end{lemma}

\subsection{Sampling}

One of the important parts of digraphons are the guarantees that a finite sampling is going to converge to the underlying digraphon. The sampling works as follows.\\

Given a digraphon $W$ and an ordered set $S=(x_1, \ldots, x_n), x_i\in[0,1]$, we create a weighted digraph $H(S, W)$ on the node set $[n]$ with the edge weights $H(i, j) = W(x_i, x_j)$. Now from such an $H$ we can create a random simple unweighted digraph by trying to sample $G$ and adding an edge $G(i, j)$ with probability $H(i, j)$. \\

For example, if $W$ is the uniform function with $W(i, j) = p, 0 \le p \le 1$, we would get the standard Erdos-Renyi graphs with probability $p$. If $W = W_G$, the canonical digraphon for a digraph  $G$, then if we sample $k$ points from $W_G$, it is the ``almost'' the same as calculating a random subgraph of $G$. It is not the same as we might have sampled $x_i, x_j$ from the same step in $W_G$. To have an exact subgraph sampling, we need to condition on the fact that $x_i, x_j$ need to be from different steps in $W_G$. In particular, we are removing sequences$(x_1, \ldots, x_n)$, with repetitions, which has $\binom{k}{2}$ such sequences, and hence a measure of $\frac{\binom{k}{2}}{n}$. This gives us that the average distance between a randomly chosen subgraph, $R(k, G)$ and a randomly sampled digraphon$R(k, W_G)$ is
\[ d(R(k, G), R(k, W_G)) \le \binom{k}{2}\frac{1}{n}\]

Now that we know how to sample, we can start looking at how sampling helps in parameter estimation. For such a scenario, we want to start with a notion of ``good'' parameters, which we can hope to be estimable. As it will turn out, most of the real world scenarios are going to be good and can be estimated using sampling. We wish to achieve some notions similar to the limit theorems for classical statistics which will give us confidence on doing real world analytics using EM or MAP algorithms for estimations.

A \emph{reasonably smooth} graph parameter is defined as a function $f(G([S]))$, which is to say of a sampling of a digraphon, which satisfies that $|f(G)-f(G')| < 1$, for two graphs $G, G'$ on the same set of nodes, whose edges differ only for a single vertex. We then achieve a sample concentration theorem

\begin{theorem}[Sample concentration theorem for digraphons\cite{lovasz2012large}]
	\hfill\\
	Let $f$ be a reasonably smooth graph parameter, and let $W$ be a digraphon, $k > 1 \in \bbN$. Let $f_0=\expct[f(R(k, W_G))]$, then for all $t > 0$, 
		\[ \Pr\left[f(R(k, W_G)) > f_0 + \sqrt{2tk}\right] < e^{-t}\]
\end{theorem}

This theorem gives credibility to the fact that our intuitive sampling algorithms are going to be working correctly on standard simulations. In fact, we can achieve an even better result which states.

\begin{theorem}[Cut distance confidence\cite{lovasz2012large}]
	Let $k > 1 \in \bbN$ and let $G$ be a digraph on $k$ nodes. Then with probability at least $1-\exp{(-\frac{k}{2\log k})}$, we have that
	\[ d_\square(G, R(k, W_G)) \le \frac{20}{\sqrt{\log k}}\]
\end{theorem}

Now with such a result, we can employ our standard statistical tools to carry out prior/posterior estimations and simulations using generative models. Indeed, we develop a maximum a posteriori (MAP) estimation algorithm for the Dirichlet priors on digraphon generative models and due to this result, we can be somewhat confident in the fact that we are achieving a good result, subject to proper maximum optimizations in the log likelihood estimation.
	\section{Evolution by Duplication}
\label{sec:dg_model}

In many scenarios, preferential attachment happens at a larger scale than a single vertex, where multiple parts, aka clusters, of the network get duplicated. This process is hard to model in a one shot setting where we have to duplicate the parts one at a time. For added robustness, we expand the definition of \emph{preferential attachment} to allow duplicating subgraphs. 

We specifically want to model graphs for evolutionary events, such as graphical models and causal networks. Each node represents an event that can take place, while each edge (potentially weighted) represents the \emph{influence} from an event to another. We call this the \emph{event graph} and represent this as a weighted adjacency matrix. An example of such a network that is the Suppes Bayes causal network (SBCN). 

We define evolution by duplication as an extension of the preferential attachment model, where we allow extending the graph by duplicating a larger subgraph. This subsumes the original case where preferential attachment of a single node and allows for a larger, more robust evolutionary model.

\begin{definition}[Evolution by Duplication]
Let $G$ be a directed graph, and let $X = [a_i,\ldots, a_j] \times [b_k,\ldots, b_l] \subset [1,\ldots,n]^2$. A \textbf{preferential attachment} of $X$ on $G$, with the weight function $\theta$, is the new digraphon $G'$, defined by
\[ G' = (1-\theta(X))\cdot G + \theta(X)\cdot G|_{X}\]
\end{definition}

Typically, we want to randomly select segments which are going to be attached. This is carried out with the help of a weight function $\theta$, which is inversely proportional to the measure of the attached segment, $\theta(X)\propto \frac{1}{\nIntegral{X}{~}{G}} = \frac{1}{\sum_{\substack{x\in[a_i,\ldots,a_j] \\ y \in [b_k,\ldots, b_l]}} w(a_x, b_y)}$. This weighting implies that smaller segments are easier to attach than larger segments, which mimics the biological characteristics, wherein it is easier to duplicate smaller, simpler networks over complex multi-path networks.

\begin{figure}[ht]
    \centering
    \begin{tikzpicture}[scale=0.5]
    \foreach \i in {\xMin,...,\xMax} {
        \draw [very thin,gray] (\i,\yMin) -- (\i,\yMax);
    }
    \foreach \i in {\yMin,...,\yMax} {
        \draw [very thin,gray] (\xMin,\i) -- (\xMax,\i);
    }
    \draw[very thick, blue] (0,0) -- (0,6) -- (6,6) -- (6,0) -- (0,0);
    \draw[very thick, yellow] (4,4) -- (4,9) -- (9,9) -- (9,4) -- (4,4);
    \draw[very thick, red] (5,0) -- (5,5) -- (9,5) -- (9,0) -- (5,0);
    \draw[very thick, green] (0,5) -- (0,9) -- (5,9) -- (5,5) -- (0,5);
    \end{tikzpicture}
    \caption[Segment basis for a digraph.]{Segment basis for a digraph.\\ The colored rectangles denote \emph{segments} for a digraph, which can be duplicated by \emph{preferential attachment}. It is not necessary for the segments to be non-intersecting, the overlapped regions will be attached multiple times through different segments. Each segment also has an associated weight, which influences the evolutionary priority of duplicating that segment.}
    \label{fig:dg_segment_digraph}
\end{figure}

Due to the exponentially large size of the segment space, it is preferred to have a smaller finite collection of segments of interest. In biological networks, such as those for cancer somatic mutations, we typically want to restrict our attention to either driver genes or those known to interact with cancer mutations up to some extent. For example, The Cancer Genome Atlas (TCGA) program or the Catalogue of Somatic Mutations in Cancer (COSMIC) databases are the main places which give information about important genes.

\begin{definition}[Segmented Digraph]
A \emph{segmented} digraph, \cref{fig:dg_segment_digraph}, is a digraph with a finite collection of weighted segments $(X_i, w_i)$, $X_i = [a_i, b_i]\times [c_i, d_i]$,  with weights $w_i\in \bbR_+, \sum w_i = 1$.
\end{definition}

We can also allow dynamic sized vertices by allowing an attachment to introduce a node, thereby adding a row or a column to our adjacency matrix. For our current analysis, we restrict ourselves to static sized attachments, which only affect weights of the graph.

It is important to note that this is an exchangeable process; relabeling the segments, $X_i$, does not lead to a different simulation, we only depend upon the weights of the segments. Thus we can leverage all the theoretical properties of digraphons, such as asserting that any such evolutionary process, wherein we do simulations either by attachments to make larger graphs, or by performing weighted boosting, is going to converge to a final digraphon.

\begin{figure}[ht]
    \centering
    \begin{tikzpicture}
    \node(a1) at (-1,0) {a};
    \node(a2) at (1,0) {b};
    \node(a3) at (-1,-1) {c};
    \node(a4) at (0,-1) {d};
    \node(a5) at (1,-1) {e};
    \node(a6) at (-1,-2) {f};
    \node(a7) at (1,-2) {g};
    \node(a8) at (0,-3) {h};
    \draw[->, thin] (a1) -- (a3);
    \draw[->, thin] (a1) -- (a4);
    \draw[->, thin] (a2) -- (a4);
    \draw[->, thin] (a2) -- (a5);
    \draw[->, thin] (a3) -- (a6);
    \draw[->, thin] (a4) -- (a6);
    \draw[->, thin] (a4) -- (a7);
    \draw[->, thin] (a5) -- (a7);
    \draw[->, thin] (a6) -- (a8);
    \draw[->, thin] (a7) -- (a8);
    
    \draw[->, thick] (1.5, -1.5) -- (3.5 ,-1.5);
    
    \node(b1) at (4,0) {a};
    \node(b2) at (6,0) {b};
    \node(b3) at (4,-1) {c};
    \node(b4) at (5,-1) {d};
    \node(b5) at (6,-1) {e};
    \node(b6) at (4,-2) {f};
    \node(b7) at (6,-2) {g};
    \node(b8) at (5,-3) {h};
    \draw[->, thin] (b1) -- (b3);
    \draw[->, thin] (b1) -- (b4);
    \draw[->, thin] (b2) -- (b4);
    \draw[->, thin] (b2) -- (b5);
    \draw[->, thin] (b3) -- (b6);
    \draw[->, very thick, red] (b4) to [bend left=20] (b6);
    \draw[->, very thick, red] (b4) to [bend right=20] (b6);
    \draw[->, very thick, red] (b4) to [bend left=20] (b7);
    \draw[->, very thick, red] (b4) to [bend right=20] (b7);
    \draw[->, thin] (b5) -- (b7);
    \draw[->, very thick, red] (b6) to [bend left=20] (b8);
    \draw[->, very thick, red] (b7) to [bend left=20] (b8);
    \draw[->, very thick, red] (b6) to [bend right=20] (b8);
    \draw[->, very thick, red] (b7) to [bend right=20] (b8);
    \end{tikzpicture}
    \caption[Evolution by subgraph attachment via duplication.]{Evolution by subgraph attachment via duplication.\\ Preferential attachment of the segment $[d,f,g,h]^2$. The same effect can be seen for more than one segment attachment, as evidenced by preferentially attaching $[d, f, g] \times [f, g, h]$.}
    \label{fig:dg_pattach}
\end{figure}

	\section{Finite Modeling and Implementations}
\label{sec:dg_impl}

To model such an evolutionary process is non trivially complex, as the graphs can have large number of nodes and edges. Naive algorithms for doing \emph{preferential attachment} of a segment $[i, j]\times [k, l]$, have a time complexity of $(j-i)\cdot(k-l) = O(n^2)$. In addition to that, to do a weighted sample, we need to get the current weight of a segment, which would also take $O(n^2)$. This strategy is feasible in cases of small segments, but has an undesirable asymptotic behaviour, which we improve upon. In this section, we detail an efficient data structure for simulating the preferential attachment framework on a digraph. 

There are two key operations that we want our model to access:

\begin{enumerate}
    \item Get the current weight of a segment $X$.
    \item \emph{Preferentially Attach} a segment, $X$, with a particular weight, $\theta(X)$.
\end{enumerate}

We note that we can perform (2) in two steps - (a) Multiply the whole graph by $(1-\theta(X))$, (b) Multiply the segment $X$ by $\frac{1}{1 - \theta(x)}$.

Hence, our data structure needs to support the following operations:

\begin{enumerate}
    \item Multiply a segment $[a, b] \times [c, d]$ by some value $c$.
    \item Return the sum of all elements in a segment $[a, b] \times [c, d]$.\\
    For a finite model implementation, we assume that the size of the digraph is fixed.
\end{enumerate}

\begin{figure}
    \centering
    \begin{tikzpicture}[scale=0.5]
    \draw[fill=yellow!90] (1,2) -- (1,3) -- (6,3) -- (6,2);
    \draw[fill=yellow!90] (1,6) -- (1,7) -- (6,7) -- (6,6);
    \draw[fill=yellow!30] (1,3) -- (1,6) -- (6,6) -- (6,3);
    
    \draw[fill=blue!90] (4,0) -- (4,3) -- (8,3) -- (8,0);
    \draw[fill=blue!30] (4,3) -- (4,4) -- (8,4) -- (8,3);
    
    \draw[fill=green!90] (4,2) -- (4,3) -- (6,3) -- (6,2);
    \draw[fill=green!30] (4,3) -- (4,4) -- (6,4) -- (6,3);

    \foreach \i in {\xMin,...,\xMax} {
        \draw [very thin,gray] (\i,\yMin) -- (\i,\yMax);
    }
    \foreach \i in {0,...,2} {
        \draw [decorate,decoration={brace,amplitude=10pt},xshift=-4pt,yshift=0pt] (-0.5,\i*3) -- (-0.5,\i*3 + 3) node [black,midway,xshift=-0.8cm] {$T_{\i}$};
        \draw [very thick, red] (0,\i*3) -- (0,\i*3 + 3) -- (\xMax,\i*3 + 3) -- (\xMax,\i*3) -- (0,\i*3);
    }
    \foreach \i in {\yMin,...,\yMax} {
        \draw [very thin,gray] (\xMin,\i) -- (\xMax,\i);
    }
    \end{tikzpicture}
    \caption[Graph data structure for preferential attachment.]{Graph data structure for preferential attachment.\\ We break the graph into $\sqrt{n}$ groups along the rows and create a segment tree, $T_i$, for each of these groups. We see that any preferential attachment that is done can be broken down into $\sqrt{n}$ parts, where we do updates for each tree, with a total of $\sqrt{n}$ updates. The overlaps are handled efficiently using lazy propagation in the segment tree, which allow us to do the weight updates in $O(\log n)$ for each individual tree.}
    \label{fig:dg_graphds}
\end{figure}

We represent the digraph $G$ using the \textbf{square-root decomposition} along the rows and a \textbf{segment tree} along the columns for each group of $\sqrt{n}$ rows, \cref{fig:dg_graphds}.\\
Each $T_i$ represents a collection of $\sqrt{n}$ rows with a segment tree. An update of a segment $X=[a,b]\times[c,d]$ can span across segment trees. The square root decomposition facilitates fast updates in $O(\sqrt{n}\log{n})$ time.

For each segment tree, it is possible to do \emph{range sum} together with \emph{range multiply} using \emph{lazy propagation} in $O(\log{n})$, \cref{algocf:dg_lazy_mult} and \cref{algocf:dg_lazy_sum}.

To prove that the whole data structure works in the time complexity described above, we break the analysis into two parts. First, we show that in an individual segment tree, we can solve range sum and range multiplication in $O(\log n)$. Then we show how to extend this across rows by grouping multiple trees together.

\subsection{Segment Tree - Lazy Multiply and Sum}

The per row simplification of this problem boils down to

\begin{problem}[Lazy Multiply and Sum]
Given an array of numbers, $\mArray{a}{1}{n}$, perform the following operations in $O(\log n)$
\begin{enumerate}
    \item Find the sum of all numbers in a contiguous range $[i, j]$.
    \item Multiply all numbers in a contiguous range $[i, j]$ by some value $k$.
\end{enumerate}
\end{problem}

\begin{figure}
    \centering
    \begin{tikzpicture}
    \node (1) at (0,0) {$\mArray{a}{1}{n} - \{sum = s_0, mult = m_0\}$};
    \node (2) at (-3,-1) {$\mArray{a}{1}{n/2} - \{ s_1, m_1\}$};
    \node (3) at (3,-1) {$\mArray{a}{n/2 + 1}{n} - \{s_2, m_2 \}$};
    \draw[->] (1) -- (2);
    \draw[->] (1) -- (3);
    \draw[->] (2) -- (-5, -2);
    \draw[->] (2) -- (-1, -2);
    \draw[->] (3) -- (1, -2);
    \draw[->] (3) -- (5, -2);
    \node at (0, -3) {$\vdots$};
    \node at (-6, -4) {\large [};
    \node at (6, -4) {\large ]};
    \draw[dotted] (-5.5, -4.2) -- (5.5, -4.2);
    \end{tikzpicture}
    \caption[Segment tree for lazy propagation.]{Segment tree for lazy propagation.\\ For each node in the segment tree, we store the subtree range along with two additional parameters, current sum and propagated multiplicand. The multiplicand is not propagated fully until a  subrange is updated in a subsequent query or until a subrange weight is queried. And even in those cases, we only propagate it optimally with lazy propagation.}
    \label{fig:dg_segmenttree}
\end{figure}

We solve this by creating a segment tree of nodes, \cref{fig:dg_segmenttree} ,with each node containing the following metadata:

\begin{itemize}
    \item \emph{low, high} - the lower and upper end points of the range of the node
    \item \emph{sum} - current sum of all values in the range of the node
    \item \emph{mult} - current multiplicand not yet propagated to lower nodes
    \item \emph{left, right} - left and right children nodes of current node
\end{itemize}

\begin{algorithm}
  \caption{Segment Tree - Lazy Multiply}
  \label{algocf:dg_lazy_mult}
  \begin{algorithmic}[1]
    \Function{LazyMultiply}{Segment Tree Node $node$, int $l$, int $r$, int $c$}
      \State\Comment{Multiplies the part of the range $[l, r]$ contained inside $node$ by $c$}
      \State\Comment{Time complexity = $O(\log{n})$}
      \If{$node.low > r$ or $node.high < l$}
      \State\Comment{No intersection, nothing to do}
      \State\Return
      \ElsIf{$node.low \ge l$ and $node.high \le r$}
      \State\Comment{Fully contained, multiply current multiplicand by $c$}
      \State $node.mult \asteq c$
      \State $node.sum \asteq c$
      \Else
      \State LazyMultiply($node.left$, $node.low$, $node.high$, $node.mult$)
      \State LazyMultiply($node.right$, $node.low$, $node.high$, $node.mult$)
      \State $node.mult = 1$
      \State LazyMultiply($node.left$, $l$, $r$, $c$)
      \State LazyMultiply($node.right$, $l$, $r$, $c$)
      \State $node.sum = node.left.sum + node.right.sum$
      \EndIf
    \EndFunction
  \end{algorithmic}
\end{algorithm}

The lazy multiplication algorithm, \cref{algocf:dg_lazy_mult}, works by manipulating the $mult$ parameter to keep track of the accumulated multiplication and only propagating it when intersecting ranges are updated. 

\begin{theorem}[Lazy Multiply]
\label{thm:dg_thm_lm}
The running time of \cref{algocf:dg_lazy_mult} is $O(\log{n})$ for all ranges $[l, r]$.
\end{theorem}

\begin{proof}
Notice that the recursion stops at any node at which the range of the node is full contained inside the range to be updated. For example, if our tree is on $[1,\ldots, 8]$, we will have $15$ total nodes, \[[1, 8], [1, 4], [5, 8], [1, 2], [3, 4], [5, 6], [7, 8], [1,1], \ldots, [8,8]\].
If we wish to do a range multiplication on $[4, 7]$, the algorithm will stop at the top most nodes possible which unify up to our desired range, which in this case - $[4,4], [5,6], [7,7]$, which is strictly smaller than $4$, the range to be updated. 

Notice that we will only ever update a maximum of two nodes of the same length, and they will never be neighbours, as the recursion would instead stop at the parent node. Hence, to represent our range $[l, r]$, as a unification of $k$ ranges - $[l, r] = \cup_{i=1}^k [x_i, y_i]$, we can only have $k < 2\log n$, as the maximum size of a node's subrange is $n/2$.

This argument shows that we will only ever visit $O(\log n)$ nodes, proving our running time proposition.
\end{proof}

We can reach the desired worst case of $\sim2\log n$ for a tree of range $[1, 2^n]$ and a range update for $[2, 2^n-1]$.

\begin{algorithm}
  \caption{Segment Tree - Lazy Sum}
  \label{algocf:dg_lazy_sum}
  \begin{algorithmic}[1]
    \Function{LazySum}{Segment Tree Node $node$, int $l$, int $r$}
      \State\Comment{Returns the sum of the part of the range $[l, r]$ in $node$}
      \State\Comment{Time complexity = $O(\log{n})$}
      \If{$node.low > r$ or $node.high < l$}
      \State\Comment{No intersection}
      \State\Return 0
      \ElsIf{$node.low \ge l$ and $node.high \le r$}
      \State\Comment{Fully contained}
      \State\Return $node.sum$
      \Else
      \State LazyMultiply($node.left$, $node.low$, $node.high$, $node.mult$)
      \State LazyMultiply($node.right$, $node.low$, $node.high$, $node.mult$)
      \State $node.mult = 1$
      \State\Return LazySum($node.left$, $l$, $r$) + LazySum($node.right$, $l$, $r$)
      \EndIf
    \EndFunction
  \end{algorithmic}
\end{algorithm}

\begin{theorem}[Lazy Sum]
The running time of \cref{algocf:dg_lazy_sum} is $O(\log n)$ all ranges $[l, r]$.
\end{theorem}

The proof is identical to that of the lazy multiplication, where we only look at the top level nodes.

\subsection{Lazy Attach}

To utilize the previous algorithms for a 2-D structure, we have to make certain modifications. \\
We first create a segment tree for each row and then group them together in groups of size $\sqrt{n}$. For each group, create a parent segment tree, with the same metadata, except that for each node, store the \textbf{parent.sum = $\sum_{i\in \text{rows}}$ row.sum}. The parent tree stores the aggregate information across all rows, which allows us to do range queries across the whole group in $O(\log n)$. \\
It is vital to note that this is only for the whole group and not a subset of the group. To query for a subset of the group, we have to go to each individual row and query its segment tree.
Due to the fact that each group is of size $\sqrt{n}$, we are guaranteed that each operations only touches a maximum of $\sqrt{n}$ groups, out of which only two groups need to ever do individual row updates, as show in \cref{algocf:dg_lazy_attach}.

\begin{algorithm}
  \caption{Preferential Attachment on a Digraph}
  \label{algocf:dg_lazy_attach}
  \begin{algorithmic}[1]
    \Function{PreferentialAttach}{Digraph $G$, Segment $X$, Weight k}
      \State\Comment{Modifies $G$ in place by attaching the 2-D segment $X = [a, b] \times [c, d]$}
      \State\Comment{Time complexity = $O(\sqrt{n}\log{n})$}
      \For{$T_i \in [1, \sqrt{n}]$}
      \State LazyMultiply($T_i.root, 1, n, 1-k$)
      \EndFor
      \For{$T_i \in [a, b]$}
      \State\Comment{These trees are \textbf{fully} inside the range and can be updated as a group}
      \State LazyMultiply($T_i.root$, $c$, $d$, k)
      \EndFor
      \For{boundary trees $T_a$ and $T_b$}
      \State\Comment{These two trees have partial intersection with the range $[a, b]$ and must be updated manually}
      \State\Comment{Each individual row is also represented as a segment tree}
      \For{$R_i \in T_a, T_b$ and $i \in [a, b]$}
      \State LazyMultiply($R_i.root$, $c$, $d$, 2)      
      \EndFor
      \EndFor
    \EndFunction
  \end{algorithmic}
\end{algorithm}

\subsubsection{Real World Optimizations}

There are some factors that can be considered for optimizations.

\begin{itemize}
    \item If it is known that segment sizes to be updated are within a certain bounded width $w$, then it is possible to create segment groups of size $w$ instead of $\sqrt{n}$, this implies, that there will only ever be $O(w\log n)$ maximum time.
    \item For really small width segments, it can also be possible to use data structures such as a Quad Tree or a k-D tree.
    \item If working with large segments, but a small number of such segments, it is beneficial to look at binary space partition (BSP) trees and pre-partition segments into non intersecting parts. This can quickly become complex, with a growth in number of segments. BSPs are used in computer vision to do segmentation, which allows us to use highly optimized implementations, if such a scenarios is feasible.
\end{itemize}

	\section{Learning and Inference}
\label{sec:dg_learn}

Sampling of a digraphon is done by the Chinese Restaurant Process(CRP), which is a staple tool to model the Dirichlet distribution\cite{teh2010dirichlet,maceachern1998estimating}. The CRP process aims to model how people are assigned tables when sitting at a shared seating restaurant. With higher probability, people wish to sit next to others for a more pleasant experience, while with a lower probability, they wish to get a new table. This is formalized as follows.

Let $\alpha\in(0,1)$ be a hyperparameter. At any point of time $n$, let us have $k$ groups of size $[g_1,\ldots,g_k]$, which are a partition of $[1, \ldots, n]$. At time $n+1$, we wish to assign a group to the element $n+1$ which is done as follows:
\begin{itemize}
    \item With probability $\frac{(1-\alpha)|g_i|}{n+1}$, $n+1$ is assigned to group $i$.
    \item With probability $\frac{\alpha}{n+1}$, $n+1$ is assigned to a new group $g_{k+1}$
\end{itemize}

As we scale $n\to\infty$, we achieve a distribution of the set $\bbN$, which is the Dirichlet distribution with scaling parameter $\alpha$. The CRP representation is very useful convenient when performing finite sampling and parameter estimation. Another advantage of the CRP representation is the ease of computation and the fast simulations, which are important for larger models.

\subsection{Sampling a Digraphon}

Similar to Bayesian statistical models, we need a generative model for a digraphon to be able to get insights using parameter estimation. The most common generative models used are the Dirichlet prior, which do leverage using the CRP model.\\
Let $\alpha$ be the hyper parameter affecting the CRP model to get the clustering assignment of the vertices and let $\beta$ be another hyper parameter for the standard Dirichlet process, such as the gamma representation.\\
The generative model for the digraphon is as follows, where we generate a digraph on $n$ vertices:

\begin{enumerate}
    \item Draw clustering assignments, $\zeta$, for each vertex,
    \[ \zeta \sim CRP(\alpha) \]
    \item Draw weights for the edges, for each pair of groups $r\ne s$
    \[ \eta_{r, s} | \beta \sim \text{Dirichlet}(\beta)\]
    \item Set the edge using the measure of the partition
    \[ \calG_{ij} = \text{Categorical}(\eta_{\zeta_i, \zeta_j)} \]
\end{enumerate}

We see that this is an exchangeable process as well, as the clustering assignments are generated irrespective of the actual labels of the vertices. This reasoning implies that as the number of vertices $n\to\infty$, the generated digraph converges to the digraphon.

\subsection{MAP inference}

Let $\calG$ be the final digraph generated from the digraphon generative model. Our aim is to infer the weights $\eta_{r,s}$.

The likelihood that $\calG$ is sampled is given by 

\[ \Pr(\calG|\zeta) = \prod_{r \ne s}(\zeta_{r, s})^{m_{r,s}}\]

where $m_{r,s}$ denotes the number of edges from cluster $r$ to cluster $s$ and we assume no self loops for simplicity.

\[ \Pr(\zeta | \alpha) = \frac{\prod_{r\ne s}(\zeta_{r,s})^{(\alpha - 1)}}{\mathbf{B}(\alpha)} \]

where $\mathbf{B}(\omega) = \frac{\prod_{i}\Gamma(\omega_i)}{\Gamma(\sum_i \omega_i)}$ is the multivariate beta function.

\begin{align*}
    \Pr(\calG|\alpha) &= \nIntegral{}{}{\Pr(\calG|\zeta)\cdot\Pr(\zeta|\alpha)} \\
    &= \nIntegral{}{}{\frac{\prod_{r \ne s}(\zeta_{r, s})^{m_{r,s}}\prod_{r \ne s}(\zeta_{r,s})^{(\alpha - 1)}}{\mathbf{B}(\alpha)}}\\
    &= \frac{1}{\mathbf{B}(\alpha)}\prod_{r\ne s}\mIntegral{}{}{\zeta_{r,s}^{m_{r,s} + \alpha - 1}}\zeta_{r,s} \\
    &= \frac{1}{\mathbf{B}(\alpha)}\prod_{r \ne s}\mathbf{B}(m_{r,s} + \alpha)
\end{align*}

We can remove the constant $\frac{1}{\mathbf{B}(\alpha)}$ and maximize the negative log likelihood \\$\sum_{r \ne s} \log\mathbf{B}(m_{r,s} + \alpha)$ at 
\[ \zeta_{r,s} = \frac{m_{r,s} + \alpha}{\displaystyle\binom{n}{2}\cdot\alpha + \sum_{r\ne s} m_{r,s}} \]

	\section{Discussion}

As final concluding remarks, we have seen how to use digraphons as generative models for directed graphs for evolutionary models. We have further developed a robust modeling data structure for fast simulations which is easily extendable for other evolutionary mechanisms, as the data structure allows for generic operations which can be used for other world models. This opens up further testing grounds for hypothesis checking by comparing simulations to real world dynamics.

There are many extensions that have yet to be explored using the theory of digraphons, of which an important one is the use of the various distance metrics on the digraphon space. There are many metrics, such as the ones induced by the $L_p$ norm , nuclear norm, and the more interesting cut-distance, which are used. Some of the more interesting questions are presented below.

(1) Can we use the metric on the digraphon space to measure similarity of models for two distinct populations?\\
An example of such a scenario would be, given data for two distinct populations, we wish to know if they have evolved from a common ancestor population. If so, how far back did they diverge? Can we quantify the divergence of populations? Even more thoroughly, can we find the evolutionary pathway used by the populations to reach the current state?

(2) Do the metrics induce a EM-type convergence for learning algorithms?\\
Given multi-dimensional data for a population, we wish to stratify it into sub-populations with individuals having a similar evolutionary patterns. This problem is reminiscent of k-means clustering where we wish to use the digraphon metric to perform the clustering. There has been extensive work in the analysis of Euclidean norms for showing convergence (if only to a local minima for the error function), which would be important to translate to the digraphon spaces. The work by Lovász, et al.,\cite{lovasz2012nondeterministic} has shown many convergence results which may be helpful for such scenarios.

	\clearpage
	\bibliography{references}
	\bibliographystyle{plain}

\end{document}